\documentclass{article}

\usepackage[preprint]{neurips_2026}
\usepackage[utf8]{inputenc}
\usepackage[T1]{fontenc}
\usepackage{hyperref}
\usepackage{url}
\usepackage{booktabs}
\usepackage{amsfonts}
\usepackage{amsmath}
\usepackage{amssymb}
\usepackage{amsthm}
\usepackage{nicefrac}
\usepackage{microtype}
\usepackage{graphicx}
\usepackage{xcolor}
\usepackage{algorithm}
\usepackage{algorithmic}
\usepackage{subcaption}
\usepackage{thmtools}
\usepackage{thm-restate}

\newtheorem{proposition}{Proposition}

\newtheorem{remark}{Remark}
\theoremstyle{definition}

\title{Gradient Boosting within a Single Attention Layer\thanks{Code available at \url{https://github.com/salehsargolzaee/boosted-attention}}}

\author{
  Saleh Sargolzaei \\
  University of Windsor \\
  \texttt{sargolz@uwindsor.ca}
}

\begin{document}

\maketitle

\begin{abstract}
Transformer attention computes a single softmax-weighted average over values---a one-pass estimate that cannot correct its own errors.
We introduce \emph{gradient-boosted attention}, which applies the principle of gradient boosting \emph{within} a single attention layer: a second attention pass, with its own learned projections, attends to the prediction error of the first and applies a gated correction.
Under a squared reconstruction objective, the construction maps onto Friedman's gradient boosting machine, with each attention pass as a base learner and the per-dimension gate as the shrinkage parameter.
We show that a single Hopfield-style update erases all query information orthogonal to the stored-pattern subspace, and that further iteration under local contraction can collapse distinct queries in the same region to the same fixed point.
We also show that separate projections for the correction pass can recover residual information inaccessible to the shared-projection approach of Tukey's twicing.
On 10M-token subsets of WikiText-103 and OpenWebText, gradient-boosted attention improves test perplexity by $6.0\%$ and $5.6\%$ over standard attention, outperforming both Twicing Attention and a parameter-matched wider baseline on both benchmarks, with two rounds capturing most of the benefit.
We further show, both theoretically and empirically, that the mechanism requires the additive residual structure of Pre-LN transformers: under Post-LN, the same architecture degrades perplexity by $9.6\%$.
\end{abstract}

\section{Introduction}
\label{sec:intro}

The attention mechanism \citep{vaswani2017attention} computes a softmax-weighted combination of value vectors, conditioned on query-key similarities.
This is a single-pass operation: the query is compared against all keys once, a probability distribution is formed, and values are averaged accordingly.
When this estimate is poor, whether because the query is ambiguous, the relevant keys are diluted among distractors, or the softmax assigns weight to incompatible values, there is no within-layer mechanism to detect or correct the error.

A parallel limitation has long been understood in classical machine learning.
A single regression tree or a single kernel smoother produces a biased estimate; the bias can be reduced by fitting a second model to the \emph{residual} of the first.
This is the insight behind gradient boosting \citep{friedman2001greedy}, which builds an additive model $F = f_0 + \eta_1 f_1 + \eta_2 f_2 + \cdots$ by sequentially fitting each $f_m$ to the negative gradient of the loss with respect to the current prediction.
With strong base learners, two or three rounds often suffice \citep{friedman2001greedy}.

We apply this principle within a single attention layer.
Round~0 produces an initial estimate via standard attention.
The residual---the difference between the input and this estimate---is then passed through a second attention round with its own learned $\mathbf{W}_Q, \mathbf{W}_K, \mathbf{W}_V$ projections, where all three projections operate on the residual signal.
A per-dimension learned gate controls how much of the correction to apply.
Crucially, the second round derives its keys and values from the residual, allowing it to operate in a different subspace than the first round and recover information that shared-kernel corrections cannot amplify.
The result is a drop-in replacement for standard attention that adds one extra set of projections and a small gating network (approximately 18\% additional parameters overall).

\paragraph{Why not simply iterate attention?}
A natural alternative is to iterate the same attention operation: given output $\hat{y}$, feed it back as the query and repeat.
This corresponds to running the modern Hopfield network \citep{ramsauer2021hopfield} toward its fixed point.
We show (Section~\ref{sec:theory}, Proposition~\ref{prop:erasure}) that this approach systematically destroys query information.
Under the local contraction conditions established by \citet{ramsauer2021hopfield}, queries in the same contraction region converge to the same fixed point, which is determined by the stored patterns and temperature alone.
This makes iteration appropriate for content-addressable memory (where the goal is to retrieve a stored pattern) but harmful for the transformer's actual task (where the query carries information that must be preserved in the output).
Our negative results---including training with Deep Equilibrium Models (DEQ) \citep{bai2019deep}---confirm that the iterative approach failed under all training procedures we tested (Section~\ref{sec:negative}).
Gradient-boosted attention avoids this failure mode by feeding a \emph{different} signal (the residual) through \emph{different} projections, rather than re-processing the same state through the same function.

\paragraph{Relation to prior work.}
The idea that transformers implement a form of gradient descent is not new.
\citet{cheng2024transformers} proved that each transformer layer implements one step of functional gradient descent in a reproducing kernel Hilbert space---which is, mathematically, one round of gradient boosting---though they did not make this connection to the boosting literature explicit.
\citet{abdullaev2025twicing} applied Tukey's twicing \citep{tukey1977exploratory} within each attention layer, smoothing the residual $\mathbf{V} - \mathbf{A}\mathbf{V}$ with the \emph{same} attention matrix $\mathbf{A}$.
Their correction reuses the same attention kernel, yielding $(2\mathbf{A} - \mathbf{A}^2)\mathbf{V}$, without learned gating or separate projections for the correction pass.
Differential Transformer \citep{ye2025differential} computes two attention maps in parallel and subtracts them, canceling shared noise; this is a parallel subtractive mechanism, while ours is sequential and error-corrective.
We discuss these and other related works in detail in Section~\ref{sec:related}.

\paragraph{Contributions.}
\begin{enumerate}
    \item We introduce gradient-boosted attention, a multi-round attention mechanism in which each round corrects the prediction error of previous rounds using separate learned projections and a per-dimension gate. The architecture maps directly onto Friedman's Multiple Additive Regression Trees (MART) framework.

    \item We identify a fundamental limitation of iterative retrieval: query information orthogonal to the stored patterns is erased after one step, and further iteration collapses distinct queries to the same fixed point (Proposition~\ref{prop:erasure}). We establish a formal correspondence to Friedman's gradient boosting under a reconstruction objective (Proposition~\ref{prop:mart}), and show that shared-projection corrections like Twicing are confined to the row span of the original value matrix and are not fitted to the residual; gradient-boosted attention's separate projections escape both constraints (Proposition~\ref{prop:separate}).

    \item We show that Post-LN placement disrupts the additive structure required by gradient boosting: the well-known rank-$(d{-}2)$ LayerNorm Jacobian \citep{xiong2020layer} couples dimensions and rescales corrections by a data-dependent factor, breaking the element-wise gate control (Proposition~\ref{prop:normalization}). Empirically, gradient-boosted attention improves perplexity by $6.0\%$ under Pre-LN but degrades it by $9.6\%$ under Post-LN, while a parameter-matched wider baseline improves under both placements.

    \item On 10M-token subsets of WikiText-103 and OpenWebText, gradient-boosted attention improves test perplexity by $6.0\%$ and $5.6\%$ over standard attention, outperforming Twicing Attention and a parameter-matched wider baseline on both benchmarks, with two rounds capturing most of the benefit.
\end{enumerate}

\section{Background}
\label{sec:background}

\subsection{Gradient Boosting}

Gradient boosting \citep{friedman2001greedy} constructs an additive model $F_M(x) = \sum_{m=0}^{M} \eta_m f_m(x)$ by sequential residual fitting.
Starting from an initial estimate $F_0(x) = f_0(x)$, each subsequent base learner $f_m$ is fit to the negative gradient of the loss:
\begin{equation}
\label{eq:boosting}
    r_m = -\frac{\partial L(y, F)}{\partial F}\bigg|_{F=F_{m-1}(x)}, \qquad f_m \approx r_m, \qquad F_m = F_{m-1} + \eta_m f_m,
\end{equation}
where $\eta_m \in (0, 1]$ is a shrinkage parameter that regularizes the update.
For the squared loss $L = \frac{1}{2}\|y - F\|^2$, the negative gradient is simply the residual $r_m = y - F_{m-1}(x)$.
A key empirical finding is that with strong base learners (e.g., deep trees), two to three rounds often capture most of the achievable improvement, with rapidly diminishing returns thereafter.

\subsection{Attention as Hopfield Retrieval}

\citet{ramsauer2021hopfield} showed that transformer attention is mathematically equivalent to one step of the modern continuous Hopfield network.
Given stored patterns $\mathbf{X} = [\mathbf{x}_1, \ldots, \mathbf{x}_N] \in \mathbb{R}^{d \times N}$ and a query $\boldsymbol{\xi} \in \mathbb{R}^d$, the update rule is:
\begin{equation}
\label{eq:hopfield}
    T(\boldsymbol{\xi}) = \mathbf{X} \, \mathrm{softmax}(\beta \mathbf{X}^\top \boldsymbol{\xi}),
\end{equation}
where $\beta > 0$ is the inverse temperature.
Setting $\beta = 1/\sqrt{d_h}$ (the per-head dimension) and allowing separate key/value projections recovers standard attention.
The associated energy function is
\begin{equation}
\label{eq:energy}
    E(\boldsymbol{\xi}) = -\frac{1}{\beta}\log\sum_{i=1}^{N} \exp(\beta \, \mathbf{x}_i^\top \boldsymbol{\xi}) + \frac{1}{2}\|\boldsymbol{\xi}\|^2 + \text{const},
\end{equation}
whose gradient yields the update rule as $\boldsymbol{\xi}_{\text{new}} = -\nabla_{\boldsymbol{\xi}} E(\boldsymbol{\xi}) + \boldsymbol{\xi} = T(\boldsymbol{\xi})$.
Ramsauer et al.\ proved that under sufficient pattern separation relative to $\beta$, the iterates $T^t(\boldsymbol{\xi})$ converge exponentially to fixed points, and identified three types: global averages over all patterns, metastable states averaging over subsets, and near-single-pattern retrieval.

\section{Method}
\label{sec:method}

\subsection{Gradient-Boosted Attention}

\begin{figure}[t]
\centering
\includegraphics[width=\textwidth]{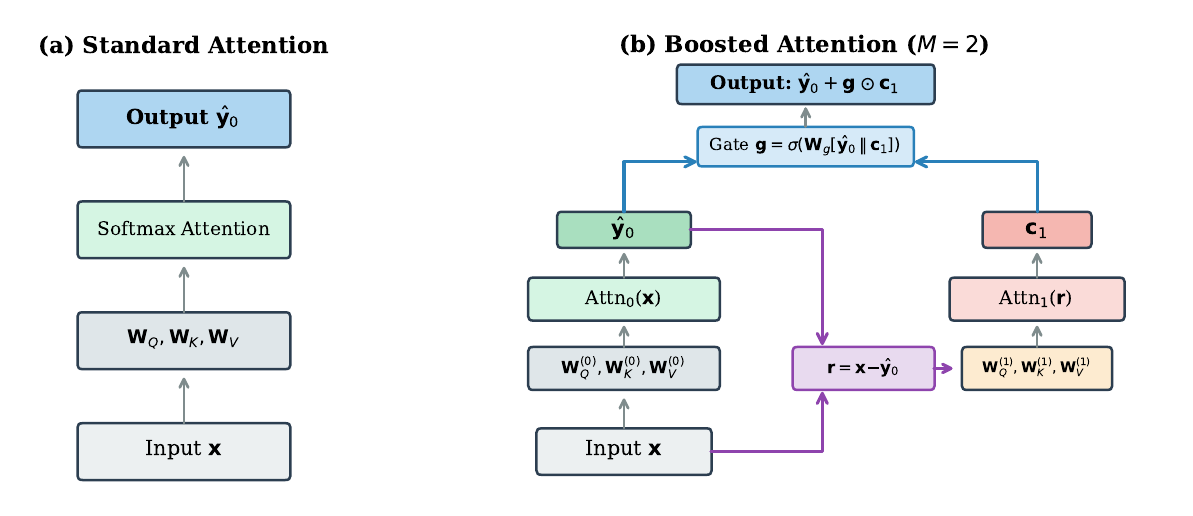}
\caption{(a)~Standard attention computes a single softmax-weighted average. (b)~Gradient-boosted attention ($M{=}2$) adds a second pass where the prediction residual $\mathbf{r} = \mathbf{x} - \hat{\mathbf{y}}_0$ is projected through separate $\mathbf{W}_Q^{(1)}, \mathbf{W}_K^{(1)}, \mathbf{W}_V^{(1)}$ to produce queries, keys, and values. A learned gate $\mathbf{g}$ controls the per-dimension correction magnitude.}
\label{fig:architecture}
\end{figure}

Let $\mathbf{x} \in \mathbb{R}^{B \times T \times d}$ denote the input to an attention layer, and let $n_h$ denote the number of attention heads with per-head dimension $d_h = d / n_h$. We define $M$ attention rounds, each with its own learned projections $\mathbf{W}_Q^{(m)}, \mathbf{W}_K^{(m)}, \mathbf{W}_V^{(m)} \in \mathbb{R}^{d_h \times d}$ (per head):
\begin{equation}
\label{eq:attn_round}
    \mathrm{Attn}_m(\mathbf{z}) = \mathrm{softmax}\!\left(\frac{\mathbf{W}_Q^{(m)} \mathbf{z} \cdot (\mathbf{W}_K^{(m)} \mathbf{z})^\top}{\sqrt{d_h}}\right) \mathbf{W}_V^{(m)} \mathbf{z}.
\end{equation}
All three projections are applied to the same input $\mathbf{z}$: the original input $\mathbf{x}$ for round~0, and the current residual $\mathbf{r}_m$ for rounds $m \geq 1$.
Each correction round thus operates in its own subspace of the residual, with keys and values that reflect what has \emph{not yet been predicted} rather than the original context.

Algorithm~\ref{alg:boosted} summarizes the forward pass.
\begin{algorithm}[t]
\caption{Gradient-Boosted Attention Forward Pass}
\label{alg:boosted}
\begin{algorithmic}[1]
\STATE \textbf{Input:} $\mathbf{x} \in \mathbb{R}^{B \times T \times d}$, number of rounds $M$
\STATE $\hat{\mathbf{y}}_0 \leftarrow \mathrm{Attn}_0(\mathbf{x})$ \hfill \COMMENT{Round 0: initial estimate}
\STATE $F \leftarrow \hat{\mathbf{y}}_0$
\FOR{$m = 1$ to $M-1$}
    \STATE $\mathbf{r}_m \leftarrow \mathbf{x} - F$ \hfill \COMMENT{Prediction error (negative gradient for $L_2$ loss)}
    \STATE $\mathbf{c}_m \leftarrow \mathrm{Attn}_m(\mathbf{r}_m)$ \hfill \COMMENT{Q, K, V all from residual via separate projections}
    \STATE $\mathbf{g}_m \leftarrow \sigma\!\left(\mathbf{W}_g^{(m)}[F \,\|\, \mathbf{c}_m]\right) \in [0,1]^d$ \hfill \COMMENT{Per-dimension shrinkage gate}
    \STATE $F \leftarrow F + \mathbf{g}_m \odot \mathbf{c}_m$ \hfill \COMMENT{Gated correction}
\ENDFOR
\STATE \textbf{Return:} $\mathbf{W}_{\mathrm{out}} \cdot F$
\end{algorithmic}
\end{algorithm}

\paragraph{The gate as a learned shrinkage parameter.}
In Friedman's framework, the shrinkage parameter $\eta_m$ is a scalar that regularizes each boosting step.
Our gate $\mathbf{g}_m = \sigma(\mathbf{W}_g^{(m)}[F \| \mathbf{c}_m]) \in [0,1]^d$ generalizes this in two ways: it is (i) per-dimension, allowing selective correction along different feature directions, and (ii) input-dependent, allowing the model to vary the correction magnitude based on the current prediction and the proposed correction.
When collapsed to a scalar, the gate reduces exactly to the shrinkage parameter.
We verify in our ablations (Section~\ref{sec:ablations}) that all gate variants (no gate, scalar, per-dimension MLP) improve over the single-round baseline, suggesting that the residual-attention mechanism is the primary ingredient.

\paragraph{Computational cost.}
With $M$ rounds, the attention computation is approximately $M$ times the cost of standard attention, plus a small gating network.
In practice, $M = 2$ adds approximately 18\% total parameters.
By contrast, Twicing Attention \citep{abdullaev2025twicing} adds zero parameters, since it reuses the same attention matrix.
Our additional cost buys separate projections for the correction pass, which Proposition~\ref{prop:separate} shows can recover information inaccessible to shared-kernel correction.
We show in Section~\ref{sec:experiments} that the improvement cannot be replicated by simply widening the standard model to match the parameter count.

\subsection{Connection to MART}
\label{sec:mart_connection}

Table~\ref{tab:mart} makes the correspondence between MART and gradient-boosted attention explicit.
Under the squared reconstruction objective $L = \frac{1}{2}\|\mathbf{x} - F\|^2$, the forward pass of Algorithm~\ref{alg:boosted} instantiates gradient boosting exactly: the residual is the negative gradient, each attention round is a base learner, and the gate is the shrinkage parameter.
Note that this is an internal objective governing the within-layer residual computation; the model is trained end-to-end with the task loss (e.g., cross-entropy for language modeling).

\begin{table}[t]
\centering
\caption{Correspondence between gradient boosting (MART) and gradient-boosted attention.}
\label{tab:mart}
\begin{tabular}{@{}lll@{}}
\toprule
\textbf{Component} & \textbf{MART} \citep{friedman2001greedy} & \textbf{Ours} \\
\midrule
Initial estimate & $f_0(x)$ (e.g., mean of targets) & $\hat{\mathbf{y}}_0 = \mathrm{Attn}_0(\mathbf{x})$ \\
Pseudo-residual & $r_m = y - F_{m-1}(x)$ & $\mathbf{r}_m = \mathbf{x} - F_{m-1}$ \\
Base learner & $f_m$ fit to $r_m$ & $\mathbf{c}_m = \mathrm{Attn}_m(\mathbf{r}_m)$ \\
Shrinkage & Scalar $\eta_m \in (0,1]$ & Gate $\mathbf{g}_m = \sigma(\mathbf{W}_g^{(m)}[\cdot]) \in [0,1]^d$ \\
Update & $F_m = F_{m-1} + \eta_m f_m$ & $F_m = F_{m-1} + \mathbf{g}_m \odot \mathbf{c}_m$ \\
\bottomrule
\end{tabular}
\end{table}

\section{Theoretical Analysis}
\label{sec:theory}

We present four results that collectively motivate and justify the gradient-boosted attention design.

\subsection{Iterating Retrieval Erases Query Information}

The most natural way to ``boost'' attention would be to iterate the same operation: compute $T(\boldsymbol{\xi})$, then $T(T(\boldsymbol{\xi}))$, and so on, converging to the Hopfield fixed point.
We state the result in the Hopfield setting (Eq.~\ref{eq:hopfield}), where the analysis is cleanest.
Part~(a) generalizes directly to standard attention with separate projections: the softmax always produces convex weights, so the output lies in $\mathrm{conv}(\mathbf{V})$ regardless of how $\mathbf{Q}$, $\mathbf{K}$, $\mathbf{V}$ are computed.
Our negative experiments (Section~\ref{sec:negative}) confirm that iteration fails in practice even with learned $\mathbf{W}_Q, \mathbf{W}_K, \mathbf{W}_V$.

\begin{proposition}[One-step projection and information loss]
\label{prop:erasure}
Let $\mathbf{X} \in \mathbb{R}^{d \times N}$ be a matrix of stored patterns and $T(\boldsymbol{\xi}) = \mathbf{X}\,\mathrm{softmax}(\beta \mathbf{X}^\top \boldsymbol{\xi})$ the Hopfield update~(\ref{eq:hopfield}).
Then:
\begin{enumerate}
    \item[(a)] For every $\boldsymbol{\xi}$, $T(\boldsymbol{\xi}) \in \mathrm{conv}(\mathbf{X}) \subseteq \mathrm{col}(\mathbf{X})$.

    \item[(b)] If $\boldsymbol{\xi} = \boldsymbol{\xi}_\parallel + \boldsymbol{\xi}_\perp$ with $\boldsymbol{\xi}_\parallel \in \mathrm{col}(\mathbf{X})$ and $\boldsymbol{\xi}_\perp \perp \mathrm{col}(\mathbf{X})$, then $T(\boldsymbol{\xi}) = T(\boldsymbol{\xi}_\parallel)$.
    Hence all information in the component orthogonal to $\mathrm{col}(\mathbf{X})$ is erased after a single step.

    \item[(c)] Any fixed point $\boldsymbol{\xi}^*$ satisfies $\boldsymbol{\xi}^* = \mathbf{X}\,\mathrm{softmax}(\beta \mathbf{X}^\top \boldsymbol{\xi}^*)$, so every fixed point lies in $\mathrm{conv}(\mathbf{X})$.
\end{enumerate}
\end{proposition}

\begin{proof}
(a) Since $\mathrm{softmax}$ returns nonnegative weights summing to one, $T(\boldsymbol{\xi}) = \sum_{i=1}^N \alpha_i \mathbf{x}_i$ with $\alpha_i \geq 0, \sum_i \alpha_i = 1$, which is a convex combination of the columns of $\mathbf{X}$.

(b) The update $T(\boldsymbol{\xi})$ depends on $\boldsymbol{\xi}$ only through the score vector $\mathbf{X}^\top \boldsymbol{\xi}$.
Since $\mathbf{X}^\top \boldsymbol{\xi}_\perp = 0$ (each column of $\mathbf{X}$ is orthogonal to $\boldsymbol{\xi}_\perp$ by definition), we have $\mathbf{X}^\top \boldsymbol{\xi} = \mathbf{X}^\top \boldsymbol{\xi}_\parallel$, so $T(\boldsymbol{\xi}) = T(\boldsymbol{\xi}_\parallel)$.

(c) Setting $\boldsymbol{\xi} = \boldsymbol{\xi}^*$ in (a) gives $\boldsymbol{\xi}^* = T(\boldsymbol{\xi}^*) \in \mathrm{conv}(\mathbf{X})$. \qed
\end{proof}

\begin{remark}
\label{rem:erasure}
Proposition~\ref{prop:erasure} reveals a structural limitation of iterative Hopfield retrieval: after one step, the state is restricted to $\mathrm{conv}(\mathbf{X})$ and depends on the query only through the score vector $\mathbf{X}^\top \boldsymbol{\xi}$.
Any query information outside $\mathrm{col}(\mathbf{X})$ is immediately and irrecoverably discarded.
Further iteration compounds this: \citet{ramsauer2021hopfield} showed (their Lemma~A7) that under sufficient pattern separation, $T$ is a local contraction within a sphere around each stored pattern, with the Banach fixed-point theorem guaranteeing a unique fixed point per sphere.
Since these fixed points are determined by $\mathbf{X}$ and $\beta$ alone, distinct queries within the same local contraction region converge to the same output, erasing within-region query information.
This is consistent with the findings of \citet{smart2025incontext}, who showed that a single attention step can realize Bayes-optimal denoising behavior that outperforms both exact retrieval of a stored pattern and convergence to a potentially spurious local minimum.
Our experiments in Section~\ref{sec:negative} confirm this empirically: iterating trained attention to convergence degrades retrieval accuracy to near chance, while the one-step output retains useful query information.
\end{remark}

\subsection{Gradient-Boosted Attention as Gradient Boosting}

\begin{proposition}[MART equivalence]
\label{prop:mart}
Consider the squared prediction loss $L(\mathbf{x}, F) = \frac{1}{2}\|\mathbf{x} - F\|^2$, where $\mathbf{x}$ is the input and $F$ is the cumulative prediction from attention rounds $0, \ldots, m-1$.
Then the negative functional gradient of $L$ with respect to $F$ is
\begin{equation}
    -\nabla_F L(\mathbf{x}, F) = \mathbf{x} - F = \mathbf{r}_m,
\end{equation}
which is exactly the residual computed in line~5 of Algorithm~\ref{alg:boosted}.
Each correction round $\mathrm{Attn}_m(\mathbf{r}_m)$ fits a base learner (attention with separate projections) to this negative gradient, and the gated update $F_m = F_{m-1} + \mathbf{g}_m \odot \mathrm{Attn}_m(\mathbf{r}_m)$ is a shrinkage-regularized gradient boosting step.
\end{proposition}

\begin{proof}
The gradient is immediate from the loss definition.
The structural correspondence between Algorithm~\ref{alg:boosted} and the boosting update~(\ref{eq:boosting}) is exact when the loss is squared and the gate $\mathbf{g}_m$ plays the role of $\eta_m$.
The only difference is that $\eta_m$ in MART is a scalar optimized via line search, while $\mathbf{g}_m$ is a per-dimension function of the current state, learned jointly with the rest of the model by backpropagation. \qed
\end{proof}

\begin{remark}
\citet{cheng2024transformers} proved that the $L$-layer transformer implements $L$ steps of functional gradient descent in a reproducing kernel Hilbert space (RKHS), but did not connect this to the classical boosting literature or to Friedman's MART.
Our construction makes this connection explicit and operates at a finer granularity: boosting within a single layer rather than across layers.
\end{remark}

\subsection{Why Separate Projections Are Necessary}

\citet{abdullaev2025twicing} proposed Twicing Attention, which smooths the residual $\mathbf{V} - \mathbf{A}\mathbf{V}$ using the \emph{same} attention matrix $\mathbf{A}$, yielding the output $(2\mathbf{A} - \mathbf{A}^2)\mathbf{V}$.
We identify two structural limitations of this shared-projection approach that gradient-boosted attention overcomes.

\begin{proposition}[Limitations of shared-projection correction]
\label{prop:separate}
Let $\mathbf{A} = \mathrm{softmax}(\mathbf{Q}\mathbf{K}^\top / \sqrt{d_h}) \in \mathbb{R}^{N \times N}$ and $\mathbf{V} = \mathbf{W}_V^{(0)}\mathbf{x} \in \mathbb{R}^{N \times d}$.
\begin{enumerate}
    \item[(a)] \emph{Row-span confinement.}
    For any polynomial $p$, each row of $p(\mathbf{A})\mathbf{V}$ is a linear combination of the rows of $\mathbf{V}$.
    In particular, the Twicing output $(2\mathbf{A} - \mathbf{A}^2)\mathbf{V}$---and any higher-degree polynomial extension---is confined to $\mathrm{rowspan}(\mathbf{V})$, the subspace spanned by the original value vectors.

    \item[(b)] \emph{Fixed correction.}
    Twicing's correction $\mathbf{A}(\mathbf{I} - \mathbf{A})\mathbf{V}$ is uniquely determined by $\mathbf{A}$ and $\mathbf{V}$ with no free parameters: in the MART framework of Proposition~\ref{prop:mart}, the weak learner is not fitted to the residual.

    \item[(c)] \emph{Gradient-boosted attention escapes both constraints.}
    The correction $\mathbf{g} \odot \mathbf{A}'\mathbf{V}'$, where $\mathbf{A}' = \mathrm{softmax}(\mathbf{Q}'\mathbf{K}'{}^\top / \sqrt{d_h})$ with $\mathbf{Q}' = \mathbf{W}_Q^{(1)}\mathbf{r}$, $\mathbf{K}' = \mathbf{W}_K^{(1)}\mathbf{r}$, $\mathbf{V}' = \mathbf{W}_V^{(1)}\mathbf{r}$, produces value vectors generically outside $\mathrm{rowspan}(\mathbf{V})$, and the projections are fitted to the residual by backpropagation.
\end{enumerate}
\end{proposition}

\begin{proof}
(a) $p(\mathbf{A})\mathbf{V} = \sum_k c_k \mathbf{A}^k\mathbf{V}$.
Each row of $\mathbf{A}\mathbf{V}$ is $\sum_j A_{ij}\mathbf{v}_j$, a linear combination of value vectors; by induction the same holds for $\mathbf{A}^k\mathbf{V}$, hence for any $p(\mathbf{A})\mathbf{V}$.

(b) The output $(2\mathbf{A} - \mathbf{A}^2)\mathbf{V}$ is a deterministic function of $\mathbf{A}$ and $\mathbf{V}$; the correction introduces no learnable parameters.

(c) Since $\mathbf{V}' = \mathbf{W}_V^{(1)}\mathbf{r}$ and $\mathbf{V} = \mathbf{W}_V^{(0)}\mathbf{x}$ use different projections of different inputs ($\mathbf{r} = \mathbf{x} - F_0 \neq \mathbf{x}$ in general), $\mathrm{rowspan}(\mathbf{V}')$ and $\mathrm{rowspan}(\mathbf{V})$ are generically distinct.
The projections $\mathbf{W}_Q^{(1)}, \mathbf{W}_K^{(1)}, \mathbf{W}_V^{(1)}$ are trained jointly, fitting the weak learner to the residual signal. \qed
\end{proof}

\begin{remark}
\citet{abdullaev2025twicing} exploit the spectral structure of part~(a) positively: they show the polynomial $\tilde{p}(\lambda) = 2\lambda - \lambda^2$ slows eigencapacity decay from $O(n^{-1})$ to $O(n^{-1/2})$ across layers, mitigating representation collapse without additional parameters.
Gradient-boosted attention addresses a complementary problem: rather than optimizing the spectral filter within $\mathrm{rowspan}(\mathbf{V})$, it introduces independent projections that fit the correction to the prediction residual, producing value vectors that can represent directions the original $\mathbf{V}$ cannot express.
\end{remark}

\subsection{Normalization Placement and the Additive Structure}
\label{sec:normalization}

Gradient boosting builds an additive model: the cumulative prediction $F_M = F_0 + \sum_{m=1}^{M} \eta_m f_m$ must preserve each correction term without distortion.
The LayerNorm Jacobian and its properties are well established \citep{xiong2020layer}; we restate the relevant facts here to make explicit their consequence for gradient-boosted attention.
Pre-LN and Post-LN differ fundamentally in whether they preserve the additive structure across layers.

\begin{proposition}[Normalization placement and additive preservation]
\label{prop:normalization}
Consider a transformer layer with sublayer output $\mathbf{f} \in \mathbb{R}^d$ (e.g., the output of gradient-boosted attention).
\begin{enumerate}
    \item[(a)] Under Pre-LN placement, $\mathbf{h}_{l+1} = \mathbf{h}_l + \mathbf{f}(\mathrm{LN}(\mathbf{h}_l))$.
    The Jacobian of $\mathbf{h}_{l+1}$ with respect to the sublayer output $\mathbf{f}$ is the identity:
    \begin{equation}
    \frac{\partial \mathbf{h}_{l+1}}{\partial \mathbf{f}} = \mathbf{I}.
    \end{equation}
    Each component of $\mathbf{f}$---including any gated correction $\mathbf{g}_m \odot \mathbf{c}_m$---contributes independently and without distortion to the residual stream.

    \item[(b)] Under Post-LN placement, $\mathbf{h}_{l+1} = \mathrm{LN}(\mathbf{h}_l + \mathbf{f}(\mathbf{h}_l))$.
    Let $\mathbf{u} = \mathbf{h}_l + \mathbf{f}$, $\mu = \frac{1}{d}\sum_j u_j$, $\sigma^2 = \frac{1}{d}\sum_j (u_j - \mu)^2$, and $\hat{\mathbf{u}} = (\mathbf{u} - \mu\mathbf{1})/\sigma$.
    The Jacobian of $\mathbf{h}_{l+1}$ with respect to $\mathbf{f}$ is the well-known LayerNorm Jacobian \citep{xiong2020layer}:
    \begin{equation}
    \label{eq:ln_jacobian}
    \frac{\partial \mathbf{h}_{l+1}}{\partial \mathbf{f}} = \frac{\mathrm{diag}(\boldsymbol{\gamma})}{\sigma}\left(\mathbf{I} - \frac{1}{d}\mathbf{1}\mathbf{1}^\top - \frac{1}{d}\hat{\mathbf{u}}\hat{\mathbf{u}}^\top\right),
    \end{equation}
    where $\boldsymbol{\gamma} \in \mathbb{R}^d$ is the learned scale parameter of LayerNorm.
    This matrix has rank $d - 2$: it annihilates the constant vector $\mathbf{1}$ and the standardized activation $\hat{\mathbf{u}}$, and it couples all surviving dimensions through the outer-product terms.
\end{enumerate}
\end{proposition}

\begin{proof}
(a) Since $\mathbf{h}_{l+1} = \mathbf{h}_l + \mathbf{f}$ and $\mathbf{h}_l$ does not depend on $\mathbf{f}$, the Jacobian is immediate.

(b) The $i$-th component of $\mathrm{LN}(\mathbf{u})$ is $\gamma_i(u_i - \mu)/\sigma + b_i$, where $b_i$ is the learned bias.
Using $\partial \mu / \partial u_j = 1/d$ and $\partial \sigma / \partial u_j = (u_j - \mu)/(d\sigma)$:
\[
\frac{\partial [\mathrm{LN}(\mathbf{u})]_i}{\partial u_j} = \frac{\gamma_i}{\sigma}\left(\delta_{ij} - \frac{1}{d} - \frac{\hat{u}_i \hat{u}_j}{d}\right).
\]
Since $\mathbf{u} = \mathbf{h}_l + \mathbf{f}$ with $\partial u_j / \partial f_j = 1$, this gives equation~(\ref{eq:ln_jacobian}).
For the rank: $\hat{\mathbf{u}}$ is centered ($\mathbf{1}^\top \hat{\mathbf{u}} = 0$), so $\mathbf{1}$ and $\hat{\mathbf{u}}$ are linearly independent.
Both are null vectors of $\mathbf{I} - \frac{1}{d}\mathbf{1}\mathbf{1}^\top - \frac{1}{d}\hat{\mathbf{u}}\hat{\mathbf{u}}^\top$ (using $\|\hat{\mathbf{u}}\|^2 = d$), and any $\mathbf{v} \perp \{\mathbf{1}, \hat{\mathbf{u}}\}$ satisfies $(\mathbf{I} - \frac{1}{d}\mathbf{1}\mathbf{1}^\top - \frac{1}{d}\hat{\mathbf{u}}\hat{\mathbf{u}}^\top)\mathbf{v} = \mathbf{v}$, giving rank exactly $d - 2$. \qed
\end{proof}

\begin{remark}
\label{rem:normalization}
Under Pre-LN, Proposition~\ref{prop:normalization}(a) ensures that the gated boosting correction $\mathbf{g}_m \odot \mathbf{c}_m$ from Algorithm~\ref{alg:boosted} enters the residual stream unchanged, preserving the additive structure $F_m = F_{m-1} + \mathbf{g}_m \odot \mathbf{c}_m$ that the MART correspondence (Proposition~\ref{prop:mart}) requires.
More broadly, the residual stream after $L$ Pre-LN layers decomposes as $\mathbf{h}_L = \mathbf{h}_0 + \sum_{l} \mathbf{f}_l$, matching the additive form of gradient boosting across layers \citep{cheng2024transformers}.
Under Post-LN, the rank-$(d{-}2)$ Jacobian projects out two degrees of freedom (the mean direction $\mathbf{1}$ and the variance-aligned direction $\hat{\mathbf{u}}$), rescales the remainder by the data-dependent factor $\boldsymbol{\gamma}/\sigma$, and introduces cross-dimensional coupling through $\hat{\mathbf{u}}\hat{\mathbf{u}}^\top$.
The per-dimension gate $\mathbf{g}_m$ was learned to apply element-wise shrinkage, but this coupling mixes the corrections across dimensions, undermining the independent per-dimension control.
We verify empirically in Section~\ref{sec:experiments} that gradient-boosted attention improves perplexity by $6.0\%$ under Pre-LN but \emph{degrades} it by $9.6\%$ under Post-LN.
This connection is consistent with the identity-mapping principle of \citet{he2016identity}, who showed in ResNets that placing normalization inside the residual branch (preserving a clean additive shortcut) improves both gradient flow and final accuracy---the same structural requirement that gradient boosting imposes.
\end{remark}

\section{Why Iterating Attention Fails: Negative Results}
\label{sec:negative}

Before presenting our main experiments, we summarize the empirical evidence that motivated the gradient-boosted attention design.
These negative results complement Proposition~\ref{prop:erasure} by showing that the failure of iterative attention persists across the training methods and architectural modifications we tested.

We trained attention models on a synthetic pattern retrieval task: $K$ unit-normalized patterns in $\mathbb{R}^d$, queries corrupted by additive Gaussian noise with standard deviation $\sigma$, evaluated by nearest-pattern retrieval accuracy (see Appendix~\ref{app:denoising} for details).
Table~\ref{tab:negative} summarizes results across six configurations.

\begin{table}[t]
\centering
\caption{Retrieval accuracy (\%) on the synthetic pattern retrieval task. The Bayes-optimal predictor (softmax attention with $\beta = 1/\sigma^2$, identity projections) provides an analytical ceiling. Both converged variants---whether trained by backpropagation through 20 iterations or by DEQ implicit differentiation---collapse to near chance, while one-step attention approaches the Bayes-optimal ceiling.}
\label{tab:negative}
\begin{tabular}{@{}lcccccc@{}}
\toprule
& \multicolumn{2}{c}{$d{=}16, K{=}4$} & \multicolumn{2}{c}{$d{=}32, K{=}8$} & \multicolumn{2}{c}{$d{=}64, K{=}16$} \\
\cmidrule(lr){2-3} \cmidrule(lr){4-5} \cmidrule(lr){6-7}
& $\sigma{=}0.5$ & $\sigma{=}0.8$ & $\sigma{=}0.5$ & $\sigma{=}0.8$ & $\sigma{=}0.5$ & $\sigma{=}0.8$ \\
\midrule
Random chance & 25.0 & 25.0 & 12.5 & 12.5 & 6.3 & 6.3 \\
Bayes optimal & 80.7 & 61.9 & 68.8 & 45.3 & 58.1 & 32.6 \\
One-step (trained) & 71.1 & 55.0 & 53.1 & 36.4 & 46.5 & 26.1 \\
Converged (trained, backprop) & 25.2 & 25.2 & 12.1 & 12.1 & 6.7 & 6.7 \\
Converged (trained, DEQ) & 25.2 & 25.2 & 12.1 & 12.1 & 6.7 & 6.7 \\
\bottomrule
\end{tabular}
\end{table}

\paragraph{Iterating attention destroys accuracy.}
Across all six configurations in Table~\ref{tab:negative}, both converged variants---trained with backpropagation through 20 unrolled iterations and with implicit differentiation via Deep Equilibrium Models \citep{bai2019deep}---achieve accuracy indistinguishable from random chance (e.g., 6.7\% vs.\ 6.3\% chance for $K{=}16$), while one-step attention reaches 26--71\% depending on difficulty, approaching the analytical Bayes-optimal ceiling.
This is consistent with Proposition~\ref{prop:erasure} and the local-contraction analysis of \citet{ramsauer2021hopfield}: queries in the same contraction region converge to the same fixed point regardless of their initial position.

\paragraph{The failure is structural, not a training artifact.}
One might suspect vanishing gradients through many iteration steps cause the converged path to fail.
We rule this out with two controls: (i)~backpropagation through the unrolled iteration graph, which provides gradients but may suffer from vanishing signal over 20 steps, and (ii)~Deep Equilibrium Models, which compute exact gradients at the fixed point via the implicit function theorem, bypassing the iteration graph entirely.
Both achieve identical random-chance accuracy (Table~\ref{tab:negative}), confirming that the information loss is a consequence of the contraction dynamics themselves, not of how gradients are computed.

\paragraph{Learned routing gates learn trivial strategies.}
We trained routing gates with five different feature sets (both outputs, their difference, scalar divergence, attention entropy) to decide when to trust the converged path.
All gates learned to always select the one-step output, achieving identical accuracy to the one-step baseline.
The converged path contains no useful complementary signal.

\section{Experiments}
\label{sec:experiments}

\subsection{WikiText-103 Language Modeling}

\paragraph{Setup.}
We train small transformer language models from scratch and compare four attention configurations: (1)~\textbf{Standard} causal multi-head attention with $d = 256$, 4 layers, and 4 heads (7.4M parameters); (2)~\textbf{Twicing} Attention \citep{abdullaev2025twicing} with the same architecture, applying the correction $(2\mathbf{A} - \mathbf{A}^2)\mathbf{V}$ at no additional parameter cost (7.4M parameters); (3)~a \textbf{parameter-fair} standard model with $d = 288$, matching the parameter count of the boosted model (8.8M parameters); and (4)~\textbf{gradient-boosted attention} ($M = 2$) with separate QKV projections per round and a per-dimension sigmoid gate (8.7M parameters).
All models use Pre-LN placement \citep{xiong2020layer}, byte-pair encoding (BPE) tokenization (16K vocabulary), sequence length 256, AdamW optimizer with learning rate $3 \times 10^{-4}$, cosine schedule with 1500 warmup steps, weight tying, and gradient clipping at 1.0.
We evaluate on two benchmarks: (i) a \textbf{10M-token subset} of WikiText-103 \citep{merity2016pointer} (the full corpus contains ${\sim}103$M tokens), and (ii) a \textbf{10M-token subset} of OpenWebText \citep{gokaslan2019openwebtext}, an open reproduction of WebText containing Reddit-upvoted web content.
Each benchmark uses a separately trained BPE tokenizer; perplexity values are comparable only within the same dataset.
We train for 15 epochs and report test perplexity averaged over 2 random seeds.
Full hyperparameters are listed in Appendix~\ref{app:hyperparams}.

\paragraph{Results.}

\begin{table}[t]
\centering
\caption{Test perplexity (lower is better) on two benchmarks. All models use Pre-LN. Perplexity values are not comparable across datasets (different tokenizers). Gradient-boosted attention outperforms all baselines on both benchmarks, including a wider model with matched parameter count.}
\label{tab:main_results}
\begin{tabular}{@{}lccc@{}}
\toprule
& & \multicolumn{2}{c}{\textbf{Test Perplexity}} \\
\cmidrule(lr){3-4}
\textbf{Model} & \textbf{Params} & \textbf{WikiText-103} & \textbf{OpenWebText} \\
\midrule
Standard ($d{=}256$)              & 7.4M & $72.2 \pm 0.3$ & $114.9 \pm 0.5$ \\
Twicing ($d{=}256$)               & 7.4M & $69.6 \pm 0.1$ & $110.7 \pm 0.3$ \\
Standard ($d{=}288$, param-fair)  & 8.8M & $69.0 \pm 0.1$ & $110.2 \pm 0.5$ \\
Gradient-boosted ($M{=}2$) & 8.7M & $\mathbf{67.9 \pm 0.1}$ & $\mathbf{108.5 \pm 0.9}$ \\
\bottomrule
\end{tabular}
\end{table}

Table~\ref{tab:main_results} shows the results.
Gradient-boosted attention achieves $67.9$ perplexity on WikiText-103 and $108.5$ on OpenWebText, the best on both benchmarks.
Relative to standard attention, this is a $6.0\%$ reduction on WikiText-103 and $5.6\%$ on OpenWebText.

Three comparisons isolate the source of improvement.
First, Twicing and our method both correct the initial attention estimate, but Twicing uses a fixed shared-kernel correction $(2\mathbf{A} - \mathbf{A}^2)\mathbf{V}$ while gradient-boosted attention uses separate projections per round with a learned gate.
The consistent gap across both datasets ($-1.7$ and $-2.2$ points, respectively) validates Proposition~\ref{prop:separate}: independent projections can fit residual structure that a shared kernel cannot.
Second, a wider standard model ($d{=}288$, $8.8$M parameters) slightly exceeds the boosted model in parameter count yet falls short on both benchmarks, isolating the contribution of the residual-fitting architecture from additional capacity alone.

\begin{figure}[t]
\centering
\includegraphics[width=\textwidth]{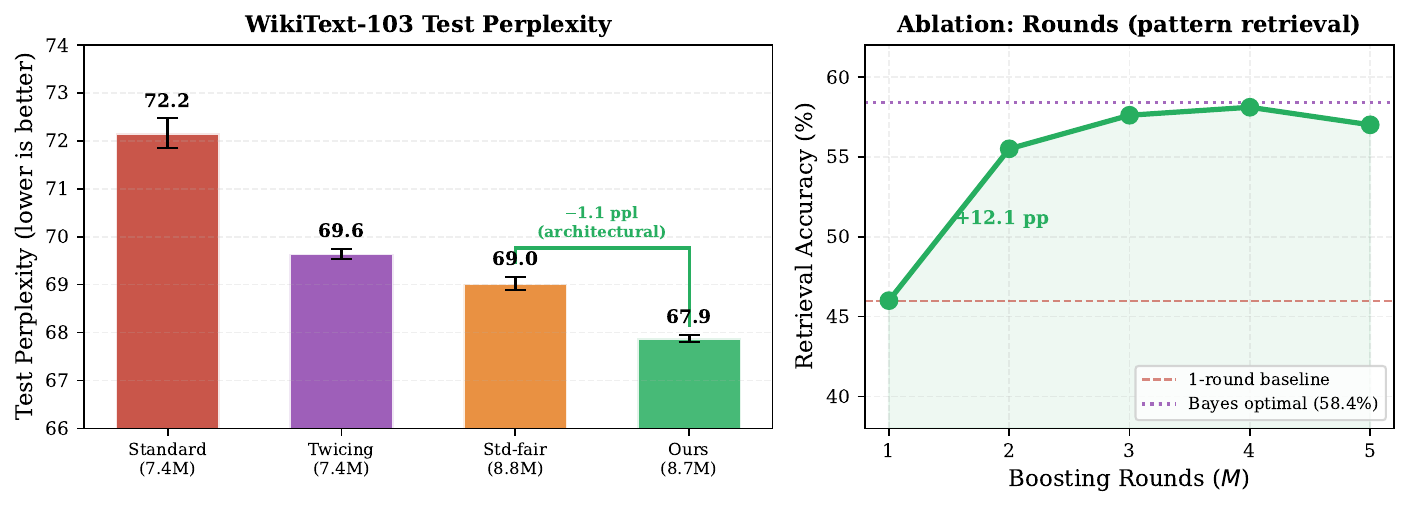}
\caption{\textbf{Left:} WikiText-103 test perplexity (zoomed axis). Gradient-boosted attention outperforms all baselines including Twicing and a parameter-matched wider model. \textbf{Right:} Retrieval accuracy on the synthetic pattern retrieval task as a function of boosting rounds. The dotted line marks the Bayes-optimal ceiling (58.1\%). Four rounds nearly match it (58.1\%); the jump from 1 to 2 rounds captures most of the improvement.}
\label{fig:results}
\end{figure}

\subsection{Normalization Ablation}
\label{sec:norm_ablation}

Proposition~\ref{prop:normalization} predicts that Post-LN placement should disrupt gradient-boosted attention by distorting the additive correction through the LayerNorm Jacobian.
To test this, we retrain three models on WikiText-103 with Post-LN placement ($\mathbf{h}_{l+1} = \mathrm{LN}(\mathbf{h}_l + f(\mathbf{h}_l))$), keeping all other hyperparameters identical (single seed).

\begin{table}[t]
\centering
\caption{Effect of normalization placement on WikiText-103 test perplexity. Under Pre-LN, gradient-boosted attention improves over standard by $6.0\%$. Under Post-LN, it degrades by $9.6\%$. The parameter-fair baseline improves under both placements, ruling out a generic optimization difficulty with extra parameters under Post-LN.}
\label{tab:normalization}
\begin{tabular}{@{}lccc@{}}
\toprule
\textbf{Model} & \textbf{Params} & \textbf{Pre-LN} & \textbf{Post-LN} \\
\midrule
Standard ($d{=}256$)              & 7.4M & $72.2$ & $77.2$ \\
Standard ($d{=}288$, param-fair)  & 8.8M & $69.0$ & $71.8$ \\
Gradient-boosted ($M{=}2$)       & 8.7M & $\mathbf{67.9}$ & $84.6$ \\
\midrule
$\Delta$ Boosted vs.\ Standard   &      & $-6.0\%$ & $+9.6\%$ \\
$\Delta$ Param-fair vs.\ Standard &     & $-4.4\%$ & $-7.0\%$ \\
\bottomrule
\end{tabular}
\end{table}

Table~\ref{tab:normalization} shows the results.
Under Post-LN, gradient-boosted attention degrades perplexity from $77.2$ to $84.6$ ($+9.6\%$), \emph{reversing} the $6.0\%$ improvement observed under Pre-LN.
The parameter-fair baseline serves as a critical control: with nearly identical parameter count ($8.8$M vs.\ $8.7$M), it \emph{improves} perplexity under Post-LN ($77.2 \to 71.8$, $-7.0\%$), ruling out the hypothesis that extra parameters are generically harder to train under Post-LN.
The failure is specific to the boosting mechanism: Post-LN's rank-$(d{-}2)$ Jacobian (Proposition~\ref{prop:normalization}(b)) distorts the gated correction before it reaches the next layer, breaking the additive structure that gradient boosting requires.
This is consistent with the observation of \citet{he2016identity} that pre-activation placement---which preserves a clean identity shortcut---is essential for depth-dependent refinement in residual networks.

\subsection{Ablation Studies}
\label{sec:ablations}

We conduct ablations on the synthetic pattern retrieval task ($d = 64$, $K = 16$ patterns, noise $\sigma = 0.5$) where the mechanism is most transparent.

\paragraph{Number of boosting rounds.}
Table~\ref{tab:rounds} shows retrieval accuracy as a function of $M$.
The Bayes-optimal predictor (58.1\%) provides an analytical ceiling.
The jump from 1 to 2 rounds ($+9.5$ percentage points) accounts for the vast majority of the improvement, bringing accuracy from 46.0\% to 55.5\%---closing most of the gap to Bayes-optimal.
Additional rounds yield diminishing returns, with 4 rounds (58.1\%) nearly matching the Bayes-optimal ceiling.
This mirrors the well-known behavior of gradient boosting with strong base learners \citep{friedman2001greedy}, where early rounds capture the dominant signal and later rounds offer marginal gains.

\begin{table}[t]
\centering
\caption{Effect of boosting rounds on pattern retrieval accuracy (\%). Synthetic pattern retrieval task, $d=64$, $K=16$, $\sigma=0.5$. Bayes-optimal ceiling: 58.1\%.}
\label{tab:rounds}
\begin{tabular}{@{}lccccc@{}}
\toprule
\textbf{Rounds ($M$)} & 1 & 2 & 3 & 4 & 5 \\
\midrule
Accuracy (\%) & 46.0 & 55.5 & 57.6 & 58.1 & 57.0 \\
$\Delta$ from $M{=}1$ & --- & +9.5 & +11.6 & +12.1 & +11.0 \\
\bottomrule
\end{tabular}
\end{table}

\paragraph{Gate type.}
We compare three gate configurations with $M{=}2$ rounds: no gate ($\mathbf{g} = \mathbf{1}$, pure additive correction), scalar gate (one learned shrinkage value per round, matching the MART $\eta_m$), and per-dimension MLP gate.
Table~\ref{tab:gate} shows that all three produce similar accuracy, with even the no-gate variant improving substantially over the single-round baseline.
That a single learned scalar performs on par with the per-dimension MLP gate suggests that the separate projections $\mathbf{W}_Q^{(1)}, \mathbf{W}_K^{(1)}, \mathbf{W}_V^{(1)}$ already absorb per-dimension adjustment, leaving the gate to control only the overall correction magnitude.
More broadly, the narrow spread across gate types confirms that the residual-attention mechanism---not the gating---is the key ingredient.

\begin{table}[t]
\centering
\caption{Effect of gate type on pattern retrieval accuracy (\%). All gated variants use $M{=}2$ rounds. Same task as Table~\ref{tab:rounds}.}
\label{tab:gate}
\begin{tabular}{@{}lcccc@{}}
\toprule
\textbf{Gate type} & None & Scalar & MLP & Baseline ($M{=}1$) \\
\midrule
Accuracy (\%) & 55.4 & 56.5 & 55.6 & 48.2 \\
$\Delta$ from baseline & +7.2 & +8.3 & +7.4 & --- \\
\bottomrule
\end{tabular}
\end{table}

\paragraph{Scaling with problem difficulty.}
The benefit of boosting grows with problem difficulty.
Across configurations ranging from $(d{=}32, K{=}8)$ to $(d{=}128, K{=}32)$, the improvement over the single-round baseline increases with the number of patterns and with moderate noise levels ($\sigma \leq 0.5$): at $\sigma{=}0.3$ the gain reaches $+17.8$ percentage points for $(d{=}64, K{=}16)$ and $+14.6$ for $(d{=}32, K{=}8)$, with $+13.5$ for $(d{=}128, K{=}32)$.
This is precisely the regime where the initial attention pass makes systematic but correctable errors---where gradient boosting is most effective (full breakdown in Figure~\ref{fig:scaling_ablation}, Appendix~\ref{app:scaling}).

\subsection{Analysis}
\label{sec:analysis}

We analyze the trained gradient-boosted models to understand what the correction round learns and where it helps most.

\paragraph{Gate values across layers.}
The gate is an input-dependent MLP that maps the current cumulative output and correction to a per-dimension mask in $[0, 1]$.
Figure~\ref{fig:gate_analysis} shows the distribution of these gate activations across dimensions, averaged over 50 test sequences, for each of the four transformer layers.
Layer~0 is the most conservative (mean gate value $0.33$, standard deviation $0.08$), admitting only a modest fraction of the correction.
Layer~1 uses the correction most aggressively (mean $0.49$) and exhibits the highest variance across dimensions ($\sigma = 0.22$), indicating that some dimensions rely heavily on the correction while others suppress it.
Layers~2 and~3 fall in between ($0.39$ and $0.36$, respectively).
When given the capacity, the gate learns non-trivial, dimension-specific shrinkage consistent with the per-dimension generalization of the MART shrinkage parameter $\eta_m$ described in Section~\ref{sec:mart_connection}.
However, as the ablation in Section~\ref{sec:ablations} shows, a single scalar gate achieves comparable accuracy---the separate projections can absorb per-dimension adjustment when the gate cannot provide it.

\begin{figure}[t]
\centering
\includegraphics[width=0.7\textwidth]{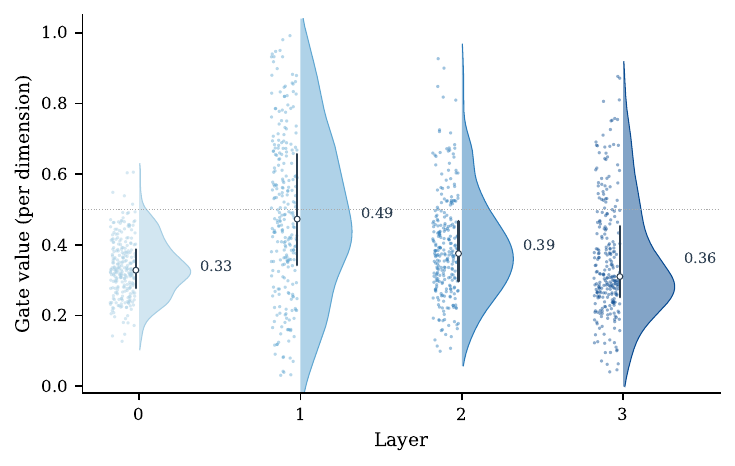}
\caption{Learned gate values per dimension for each transformer layer, averaged over 50 test sequences. The dashed line marks $g = 0.5$. Gate magnitudes and variation differ across layers, with layer~1 applying the strongest and most selective correction.}
\label{fig:gate_analysis}
\end{figure}

\paragraph{Convex hull escape.}
Proposition~\ref{prop:erasure}(a) establishes that a single attention step maps every query into $\mathrm{conv}(\mathbf{V}^{(0)})$, the convex hull of the round-0 value vectors, since the softmax produces nonnegative weights summing to one.
Proposition~\ref{prop:separate} predicts that the gated correction can move the output \emph{outside} this hull, because the correction values $\mathbf{V}^{(1)} = \mathbf{W}_V^{(1)}\mathbf{r} \in \mathbb{R}^{T \times d_h}$ span a different subspace than $\mathbf{V}^{(0)} = \mathbf{W}_V^{(0)}\mathbf{x} \in \mathbb{R}^{T \times d_h}$.
We verify this quantitatively: for each attention head at a given position $t$, we compute the $\ell_2$ distance in $\mathbb{R}^{d_h}$ from the boosted output to the nearest point in $\mathrm{conv}(\mathbf{v}_1^{(0)}, \ldots, \mathbf{v}_t^{(0)})$---the convex hull of the causally visible round-0 value vectors for that head---by solving a convex quadratic program over the probability simplex (details in Appendix~\ref{app:convex_hull}).
A distance of zero means the output can be expressed as a convex combination of the round-0 values; any positive distance indicates escape.

Figure~\ref{fig:convex_hull} shows the distribution of these distances across 600 randomly sampled (position, head) pairs per layer.
Every measured output lies strictly outside $\mathrm{conv}(\mathbf{V}^{(0)})$ across all four layers (100\% escape rate).
The per-layer ranking mirrors the gate analysis: layer~1, which applies the strongest correction (mean gate~$0.49$), escapes farthest (mean $\ell_2$ distance~$5.19$), while layer~0, the most conservative (mean gate~$0.33$), stays closest (mean distance~$0.88$).
Layers~2 and~3 fall in between ($4.25$ and $3.41$, respectively).
This confirms that the correction round's separate projections produce representations outside the original value space, as predicted by Proposition~\ref{prop:separate}.

\begin{figure}[t]
\centering
\includegraphics[width=0.7\textwidth]{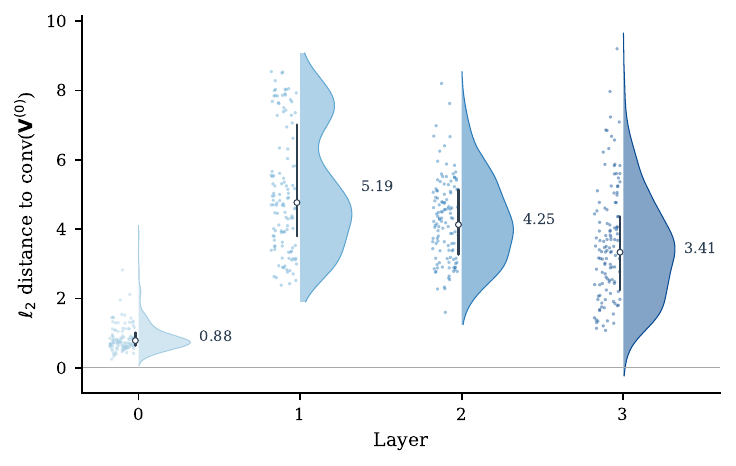}
\caption{$\ell_2$ distance in $\mathbb{R}^{d_h}$ ($d_h{=}64$) from the boosted output to the convex hull of round-0 value vectors $\mathrm{conv}(\mathbf{v}_1^{(0)}, \ldots, \mathbf{v}_t^{(0)})$, measured per attention head per position. Zero indicates the output lies within the hull; all 2{,}400 measured outputs are strictly positive, confirming escape at every layer. Annotated values are per-layer means. The ranking mirrors gate magnitudes (Figure~\ref{fig:gate_analysis}).}
\label{fig:convex_hull}
\end{figure}

\paragraph{Attention entropy.}
Figure~\ref{fig:attention_entropy} compares the entropy of each query's attention distribution across three settings: standard attention, and rounds~0 and~1 of the boosted model.
The comparison reveals a learned division of labor.
Round~0 of the boosted model is consistently \emph{more diffuse} than standard attention (mean 3.34 vs.\ 2.71 nats), while round~1 is more focused (2.55 nats).
The boosted model's initial round spreads attention more broadly than a standalone model would, delegating the sharpening to the correction round.
This effect is strongest in layer~1, where round~0 entropy exceeds standard by 27\% while round~1 drops 45\% below standard, and coincides with the highest gate values (Figure~\ref{fig:gate_analysis}).
The correlation between attention sharpness and gate openness suggests that the model learns to rely on the correction most when it can be made precise.

\begin{figure}[t]
\centering
\includegraphics[width=\textwidth]{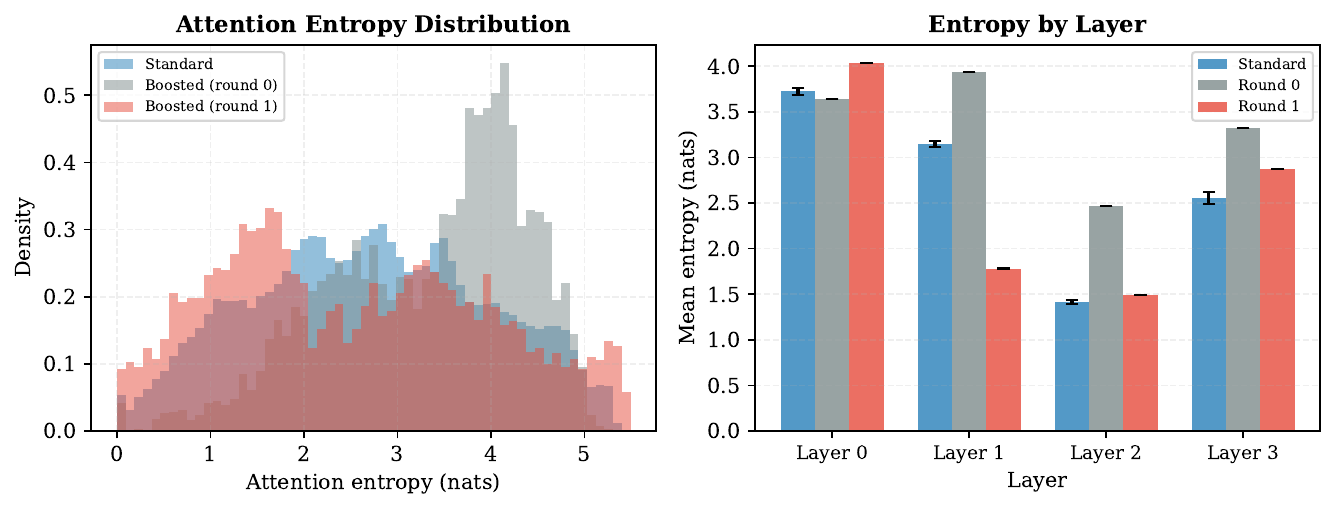}
\caption{\textbf{Left:} Distribution of attention entropy across all layers and heads for standard attention, boosted round~0, and boosted round~1. Round~0 is more diffuse than standard; round~1 is more focused. \textbf{Right:} Mean entropy per layer. The boosted model learns a division of labor: round~0 casts a wider net than standard attention, while round~1 sharpens, especially in layers~1--2.}
\label{fig:attention_entropy}
\end{figure}

\paragraph{Example-level corrections.}
Figure~\ref{fig:examples} illustrates this division of labor on individual predictions, showing the head-averaged attention in layer~1 for standard attention, boosted round~0, and boosted round~1.
We select two tokens from different articles where the boosted model dramatically outperforms standard attention.
In each case, standard attention and round~0 spread weight broadly across context tokens, while round~1 concentrates sharply on the tokens most relevant to the target.
For instance, when predicting the continuation of ``Ke'' (target: ``iser,'' completing the name Keiser), the correction round places its highest weight on ``Ke'' and the nearby geographical context, while standard attention predicts ``ong'' (confused by the earlier substring ``Yongsan'').
Similarly, for ``iron @-@ h'' $\to$ ``ul'' (completing ``hull''), the correction round attends sharply to ``iron,'' the hyphen, and ``h,'' while standard attention predicts the wrong subword.
These examples are selected from the top of the per-token improvement distribution (improvement $>4.8$ nats) and thus illustrate the best-case rather than the typical behavior; the overall improvement of $4.3$ perplexity points reflects the average across all tokens.

\begin{figure}[t]
\centering
\includegraphics[width=0.85\textwidth]{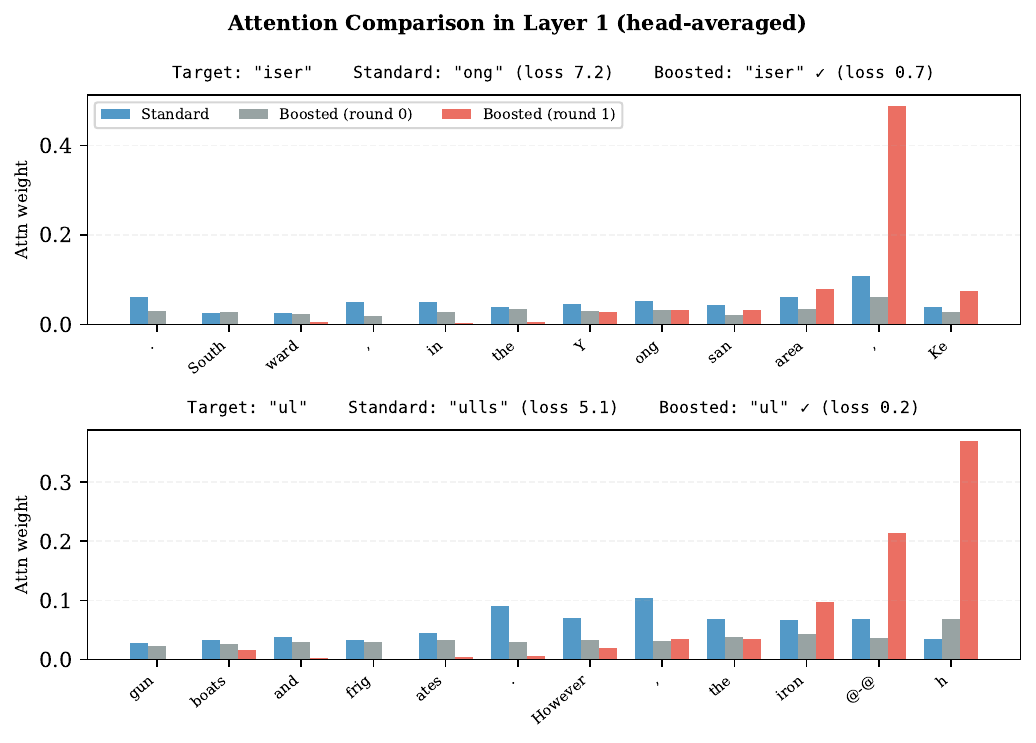}
\caption{Two tokens where gradient-boosted attention corrects a prediction error. Blue bars show standard attention, gray bars show boosted round~0, and red bars show boosted round~1. Standard and round~0 spread weight broadly; round~1 concentrates on the most relevant context. Each title shows the target, the standard model's prediction, and the boosted model's prediction with cross-entropy loss. Layer~1, head-averaged.}
\label{fig:examples}
\end{figure}

\section{Related Work}
\label{sec:related}

\paragraph{Attention and Hopfield networks.}
\citet{ramsauer2021hopfield} established the equivalence between transformer attention and one step of the modern continuous Hopfield network, identifying three types of fixed points (global average, metastable states, single-pattern retrieval).
\citet{smart2025incontext} showed that a single trained attention step implements a gradient descent update on a context-aware dense associative memory energy landscape, and that this one-step estimate can be closer to the Bayes-optimal denoiser than the converged fixed point.
Our work extends these findings by proving why convergence fails (Proposition~\ref{prop:erasure}) and proposing a constructive alternative (gradient-boosted attention).

\paragraph{Transformers as gradient descent.}
\citet{cheng2024transformers} proved that each transformer layer implements one step of functional gradient descent in the RKHS associated with the attention kernel.
This result strengthens the view that transformers can implement stagewise functional optimization closely related to boosting, though Cheng et al.\ did not make this connection to the boosting literature explicit.
\citet{huang2018learning} showed that ResNet blocks can be interpreted as boosting stages, and \citet{siu2019residual} formalized this analogy.
\citet{badirli2020grownet} used shallow neural networks as base learners in a gradient boosting ensemble.
Our work differs from all of the above by applying boosting within a single attention operation, at a finer granularity than the layer level.

\paragraph{Residual correction in attention.}
The closest prior work is Twicing Attention \citep{abdullaev2025twicing}, which applies Tukey's twicing \citep{tukey1977exploratory} within each attention layer.
Their correction smooths the residual $\mathbf{V} - \mathbf{A}\mathbf{V}$ with the \emph{same} attention matrix $\mathbf{A}$, yielding $(2\mathbf{A} - \mathbf{A}^2)\mathbf{V}$.
The theoretical justification is from nonparametric statistics: twicing reduces the bias of the Nadaraya-Watson estimator \citep{newey2004twicing}.
Our approach differs in three ways: (i) we apply separate learned projections to the \emph{residual}, producing new queries, keys, \emph{and} values that can span a different subspace than the original attention (Proposition~\ref{prop:separate}); (ii) we include a learned gate that adapts the correction magnitude per-input and per-dimension; (iii) our framing as gradient boosting provides a different theoretical lens and suggests natural extensions (more rounds, adaptive halting).
A natural question is whether the fixed correction $(2\mathbf{A} - \mathbf{A}^2)\mathbf{V}$ is optimal, or whether learning separate projections for the correction pass can do better; our experiments in Section~\ref{sec:experiments} address this directly.

\paragraph{Attention variants.}
Differential Transformer \citep{ye2025differential} computes two parallel attention maps and takes their difference, canceling common-mode noise.
This is a parallel subtractive mechanism, orthogonal to our sequential error-corrective one; the two could in principle be combined.
Gated Attention \citep{qiu2025gated} adds a post-attention sigmoid gate per head, introducing non-linearity and sparsity into a single attention pass; it does not compute a second pass or attend to residuals.

\paragraph{Iterative computation in transformers.}
Universal Transformers \citep{dehghani2019universal} iterate the same transformer block with shared parameters and adaptive halting.
Deep Equilibrium Models \citep{bai2019deep} find the fixed point of a single layer via implicit differentiation.
PonderNet \citep{banino2021pondernet} learns when to stop iterating.
All iterate the \emph{same} function on the accumulated state.
Gradient-boosted attention differs by operating on the \emph{residual} with \emph{different} projections---a distinction that Proposition~\ref{prop:erasure} shows is critical.

\paragraph{Normalization placement and residual structure.}
\citet{he2016identity} showed that pre-activation placement in ResNets preserves a clean identity shortcut, improving gradient flow and accuracy---the direct precursor of Pre-LN in transformers.
\citet{xiong2020layer} formally analyzed Pre-LN vs.\ Post-LN, proving that Post-LN produces gradient magnitudes that grow with layer index at initialization, while Pre-LN keeps them bounded.
\citet{elhage2021mathematical} formalized the transformer residual stream as an additive accumulator: each head and layer writes an independent contribution that is summed into a shared communication channel.
This additive decomposition---which Pre-LN preserves and Post-LN disrupts---is precisely the structure that gradient boosting requires.
Our Proposition~\ref{prop:normalization} connects these observations: the identity Jacobian under Pre-LN preserves boosted corrections, while Post-LN's rank-$(d{-}2)$ Jacobian distorts them.

\paragraph{Cross-layer residual methods.}
Several recent works improve information flow across transformer depth: DeepCrossAttention \citep{heddes2025deepcrossattention} replaces residual connections with depth-wise cross-attention, and Attention Residuals \citep{kimi2026attnres} replaces fixed skip connections with learned softmax attention over all preceding layer outputs, deployed at 48B-parameter scale.
These methods operate across layers; our approach operates within a single attention computation.

\section{Discussion}
\label{sec:discussion}

\paragraph{Limitations.}
Our experiments use small models (7--9M parameters) trained on two benchmarks with 10M tokens each.
While the param-fair comparison controls for capacity and the second benchmark confirms generalization, we have not yet established whether the improvement persists at the 100M--1B scale where most modern attention variants are evaluated \citep{ye2025differential}.
The additional computational cost of two attention passes per layer may be a concern for latency-sensitive applications, though the key-value computations could be shared to reduce overhead.
The Post-LN ablation uses a single seed; while the effect is large ($+9.6\%$ degradation), a multi-seed confirmation would strengthen the claim.

\paragraph{What scaling would show.}
The central question for future work is whether the 1.6\% relative improvement over a parameter-matched baseline holds, grows, or shrinks at larger model and data scales.
Gradient boosting typically helps most when individual learners are moderately strong---too weak and the residual is too noisy to fit, too strong and a single learner already captures most of the signal.
Attention in small models may be in the ``moderately strong'' regime where boosting is maximally effective; whether this persists at scale is an empirical question.

\paragraph{Future work.}
Natural extensions include: (i) applying gradient-boosted attention as a drop-in replacement during fine-tuning of pretrained models, which requires no large-scale pretraining; (ii) combining with Differential Transformer for joint noise cancellation and error correction; (iii) selectively applying boosting only at layers that benefit most, since the gate analysis (Section~\ref{sec:analysis}) shows that some layers barely use the correction while others rely on it heavily; (iv) analyzing whether the correction round develops interpretable specialization across heads and layers.

\paragraph{Connection to normalization placement.}
The finding that gradient-boosted attention requires Pre-LN placement connects our within-layer mechanism to a broader structural principle.
\citet{cheng2024transformers} proved that Pre-LN transformer layers implement functional gradient descent across depth---mathematically, gradient boosting across layers.
Our mechanism adds boosting within each layer, at a finer granularity.
Both levels of boosting require the same structural precondition: an additive residual stream where corrections are preserved without distortion.
Pre-LN satisfies this at both levels simultaneously; Post-LN disrupts it at both.
The historical shift from Post-LN \citep{vaswani2017attention} to Pre-LN, typically attributed to training stability \citep{xiong2020layer}, may also be understood as enabling the additive refinement dynamics that gradient boosting formalizes.

\paragraph{Broader motivation.}
Our architecture was partly motivated by the contrast between single-pass attention (a fast, parallel retrieval) and iterative convergence (a slower, sequential commitment to a single pattern)---a distinction reminiscent of dual-process theories in cognitive science.
The formal contribution, however, is the connection to gradient boosting, which provides precise theoretical tools rather than informal analogy.

\section{Conclusion}
\label{sec:conclusion}

We have introduced gradient-boosted attention, a mechanism that applies gradient boosting within a single attention layer.
A second attention pass, with its own learned projections, attends to the prediction error of the first pass and applies a gated correction.
We showed that the natural alternative---iterating the same attention operation---erases all query information orthogonal to the stored-pattern subspace, and under local contraction can collapse distinct queries in the same region to the same fixed point.
The formal correspondence to Friedman's MART framework connects a growing body of theoretical work on transformers-as-gradient-descent to the classical boosting literature.
Experiments on WikiText-103 and OpenWebText confirm that gradient-boosted attention outperforms standard attention, Twicing Attention, and a parameter-matched wider baseline on both benchmarks.
The normalization ablation reveals a structural requirement: the mechanism depends on Pre-LN's additive residual stream and degrades under Post-LN, consistent with the LayerNorm Jacobian analysis and with the broader principle that gradient boosting requires additive preservation of corrections.

\bibliography{references}

@article{friedman2001greedy,
  title={Greedy function approximation: a gradient boosting machine},
  author={Friedman, Jerome H},
  journal={Annals of Statistics},
  volume={29},
  number={5},
  pages={1189--1232},
  year={2001},
  publisher={Institute of Mathematical Statistics}
}

@inproceedings{vaswani2017attention,
  title={Attention is all you need},
  author={Vaswani, Ashish and Shazeer, Noam and Parmar, Niki and Uszkoreit, Jakob and Jones, Llion and Gomez, Aidan N and Kaiser, {\L}ukasz and Polosukhin, Illia},
  booktitle={Advances in Neural Information Processing Systems},
  volume={30},
  year={2017}
}

@inproceedings{ramsauer2021hopfield,
  title={Hopfield networks is all you need},
  author={Ramsauer, Hubert and Sch{\"a}fl, Bernhard and Lehner, Johannes and Seidl, Philipp and Widrich, Michael and Adler, Thomas and Gruber, Lukas and Holzleitner, Markus and Pavlovi{\'c}, Milena and Sandve, Geir Kjetil and others},
  booktitle={International Conference on Learning Representations},
  year={2021}
}

@inproceedings{cheng2024transformers,
  title={Transformers implement functional gradient descent to learn non-linear functions in context},
  author={Cheng, Xiang and Chen, Yuxin and Sra, Suvrit},
  booktitle={International Conference on Machine Learning},
  pages={8002--8037},
  year={2024}
}

@inproceedings{huang2018learning,
  title={Learning deep {ResNet} blocks sequentially using boosting theory},
  author={Huang, Furong and Ash, Jordan and Langford, John and Schapire, Robert},
  booktitle={International Conference on Machine Learning},
  year={2018}
}

@article{siu2019residual,
  title={Residual networks behave like boosting algorithms},
  author={Siu, Chapman},
  journal={arXiv preprint arXiv:1909.11790},
  year={2019}
}

@inproceedings{badirli2020grownet,
  title={Gradient boosting neural networks: {GrowNet}},
  author={Badirli, Sarkhan and Liu, Xuanqing and Xing, Zhengming and Bhatt, Avradeep and Cetin, Alina and Singh, Mononito},
  journal={arXiv preprint arXiv:2002.07971},
  year={2020}
}

@inproceedings{abdullaev2025twicing,
  title={Transformer meets twicing: Harnessing unattended residual information},
  author={Abdullaev, Laziz and Nguyen, Tan M},
  booktitle={International Conference on Learning Representations},
  year={2025}
}

@inproceedings{ye2025differential,
  title={Differential Transformer},
  author={Ye, Tianzhu and Dong, Li and Xia, Yuqing and Sun, Yutao and Zhu, Yi and Huang, Gao and Wei, Furu},
  booktitle={International Conference on Learning Representations},
  year={2025}
}

@article{qiu2025gated,
  title={Gated attention for large language models: Non-linearity, sparsity, and attention-sink-free},
  author={Qiu, Zhenyu and others},
  journal={arXiv preprint arXiv:2505.06708},
  year={2025}
}

@inproceedings{dehghani2019universal,
  title={Universal transformers},
  author={Dehghani, Mostafa and Gouws, Stephan and Vinyals, Oriol and Uszkoreit, Jakob and Kaiser, {\L}ukasz},
  booktitle={International Conference on Learning Representations},
  year={2019}
}

@inproceedings{bai2019deep,
  title={Deep equilibrium models},
  author={Bai, Shaojie and Kolter, J Zico and Koltun, Vladlen},
  booktitle={Advances in Neural Information Processing Systems},
  year={2019}
}

@article{banino2021pondernet,
  title={Ponder{N}et: Learning to ponder},
  author={Banino, Andrea and Balaguer, Jan and Blundell, Charles},
  journal={arXiv preprint arXiv:2106.01345},
  year={2021}
}

@inproceedings{smart2025incontext,
  title={In-context denoising with one-layer transformers: Connections between attention and associative memory retrieval},
  author={Smart, Matthew and Bietti, Alberto and Sengupta, Biswarup},
  booktitle={International Conference on Machine Learning},
  pages={55950--55971},
  year={2025}
}

@article{heddes2025deepcrossattention,
  title={Deep{C}ross{A}ttention: Supercharging transformer residual connections},
  author={Heddes, Lucas and others},
  journal={arXiv preprint arXiv:2502.06785},
  year={2025}
}

@article{kimi2026attnres,
  title={Attention Residuals},
  author={{Kimi Team}},
  journal={arXiv preprint arXiv:2603.15031},
  year={2026}
}

@book{tukey1977exploratory,
  title={Exploratory Data Analysis},
  author={Tukey, John W},
  year={1977},
  publisher={Addison-Wesley}
}

@article{newey2004twicing,
  title={Twicing kernels and a small bias property of semiparametric estimators},
  author={Newey, Whitney K and Hsieh, Fushing and Robins, James M},
  journal={Econometrica},
  volume={72},
  number={3},
  pages={947--962},
  year={2004}
}

@inproceedings{xiong2020layer,
  title={On layer normalization in the transformer architecture},
  author={Xiong, Ruibin and Yang, Yunchang and He, Ji and Zheng, Kai and Zheng, Shuxin and Xing, Chen and Zhang, Huishuai and Lan, Yanyan and Wang, Liwei and Liu, Tie-Yan},
  booktitle={International Conference on Machine Learning},
  pages={10524--10533},
  year={2020}
}

@inproceedings{he2016identity,
  title={Identity mappings in deep residual networks},
  author={He, Kaiming and Zhang, Xiangyu and Ren, Shaoqing and Sun, Jian},
  booktitle={European Conference on Computer Vision},
  pages={630--645},
  year={2016}
}

@article{elhage2021mathematical,
  title={A mathematical framework for transformer circuits},
  author={Elhage, Nelson and Nanda, Neel and Olsson, Catherine and Henighan, Tom and Joseph, Nicholas and Mann, Ben and Askell, Amanda and Bai, Yuntao and Chen, Anna and Conerly, Tom and others},
  journal={Transformer Circuits Thread},
  year={2021}
}

@article{gokaslan2019openwebtext,
  title={{OpenWebText} corpus},
  author={Gokaslan, Aaron and Cohen, Vanya},
  year={2019},
  url={http://Skylion007.github.io/OpenWebTextCorpus}
}

@article{merity2016pointer,
  title={Pointer sentinel mixture models},
  author={Merity, Stephen and Xiong, Caiming and Bradbury, James and Socher, Richard},
  journal={arXiv preprint arXiv:1609.07843},
  year={2016}
}
\bibliographystyle{plainnat}

\newpage
\appendix

\section{Hyperparameters and Training Details}
\label{app:hyperparams}

Table~\ref{tab:hyperparams} lists all hyperparameters for the language modeling experiments.
All models share the same training configuration; only the attention mechanism and normalization placement differ.
OpenWebText experiments use the same hyperparameters with a separately trained BPE tokenizer.

\begin{table}[h]
\centering
\caption{Hyperparameters for language modeling experiments.}
\label{tab:hyperparams}
\begin{tabular}{@{}ll@{}}
\toprule
\textbf{Hyperparameter} & \textbf{Value} \\
\midrule
Training tokens & 10M (subset of WikiText-103) \\
Vocabulary & BPE, 16,384 types \\
Sequence length & 256 \\
Layers & 4 \\
Model dimension $d$ & 256 (288 for param-fair) \\
Attention heads & 4 \\
FFN hidden dimension & $4d$ \\
Boosting rounds $M$ & 2 (for gradient-boosted attention) \\
Gate architecture & Linear($2d \to d$) + Sigmoid \\
\midrule
Optimizer & AdamW \\
Learning rate & $3 \times 10^{-4}$ \\
Weight decay & 0.01 \\
LR schedule & Cosine decay with linear warmup \\
Warmup steps & 1,500 \\
Epochs & 15 \\
Batch size & 32 \\
Gradient clipping & 1.0 \\
Dropout & 0.1 \\
Weight tying & Yes (embedding = output projection) \\
\midrule
Hardware & 1$\times$ NVIDIA RTX 2000 Ada (16GB) (WikiText-103) \\
 & 1$\times$ NVIDIA A100 80GB (OpenWebText, Post-LN) \\
 & Apple M1 (MPS) (synthetic ablations) \\
Training time per run & ${\sim}$30 min (RTX 2000), ${\sim}$10 min (A100) \\
Random seeds & 42, 123 \\
\bottomrule
\end{tabular}
\end{table}

\section{Convergence Speed}
\label{app:convergence}

Figure~\ref{fig:convergence} plots validation perplexity on WikiText-103 across epochs for all four models (averaged over two seeds).
Gradient-boosted attention reaches the parameter-fair baseline's final perplexity ($69.6$) by epoch~12, while the baseline itself requires all 15 epochs---20\% fewer optimizer steps.
It reaches the standard model's final perplexity ($72.2$) by epoch~10, a 33\% reduction.
Note that each training step is more expensive for the boosted model due to the second attention round; this comparison measures convergence in optimizer steps, not FLOPs or wall-clock time.

\begin{figure}[h]
\centering
\includegraphics[width=0.65\textwidth]{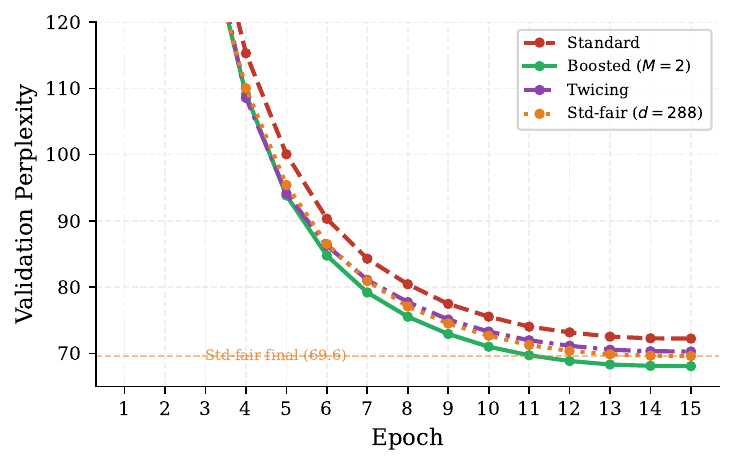}
\caption{Validation perplexity on WikiText-103 across training epochs (mean of 2 seeds). The dashed line marks the parameter-fair baseline's final perplexity. Gradient-boosted attention reaches this level 3 epochs earlier (20\% fewer optimizer steps), despite starting from a comparable first epoch.}
\label{fig:convergence}
\end{figure}

\section{Synthetic Pattern Retrieval Task Details}
\label{app:denoising}

The negative results (Section~\ref{sec:negative}) and ablation studies (Section~\ref{sec:ablations}) use a synthetic pattern retrieval task.
$K$ unit-normalized patterns $\mathbf{p}_1, \ldots, \mathbf{p}_K \in \mathbb{R}^d$ are sampled uniformly from the unit sphere.
A query is generated by selecting a pattern $\mathbf{p}_j$ uniformly at random and adding isotropic Gaussian noise: $\tilde{\mathbf{x}} = \mathbf{p}_j + \boldsymbol{\varepsilon}$, $\boldsymbol{\varepsilon} \sim \mathcal{N}(\mathbf{0}, \sigma^2 \mathbf{I})$.
Retrieval accuracy is the fraction of queries for which the nearest stored pattern (by Euclidean distance) to the model's output matches the generating pattern.
All synthetic experiments use mean squared error $\|\hat{\mathbf{x}} - \mathbf{p}_j\|^2$ as the training loss---consistent with the squared reconstruction objective underlying the MART framework (Section~\ref{sec:theory})---with Adam optimizer, learning rate $3 \times 10^{-3}$, batch size 512, and 150 training epochs.
Under this setup the task admits a closed-form Bayes-optimal predictor \citep{smart2025incontext}: $f_{\mathrm{opt}}(\tilde{\mathbf{x}}) = \sum_k \mathbf{p}_k \exp(\langle \mathbf{p}_k, \tilde{\mathbf{x}} \rangle / \sigma^2) / \sum_k \exp(\langle \mathbf{p}_k, \tilde{\mathbf{x}} \rangle / \sigma^2)$, which is softmax attention over the patterns with temperature $\beta = 1/\sigma^2$ and identity projections.
We report this baseline alongside the trained models.

\subsection{Scaling Across Configurations}
\label{app:scaling}

Figure~\ref{fig:scaling_ablation} shows the full scaling ablation across three dimension/pattern-count settings and multiple noise levels.
The improvement from boosting is largest at moderate noise ($\sigma \leq 0.5$) and higher dimensionality, where the initial attention pass makes systematic but correctable errors.
At very high noise ($\sigma \geq 0.8$), even the base learner's errors become difficult to correct, and gains diminish.

\begin{figure}[h]
\centering
\includegraphics[width=\textwidth]{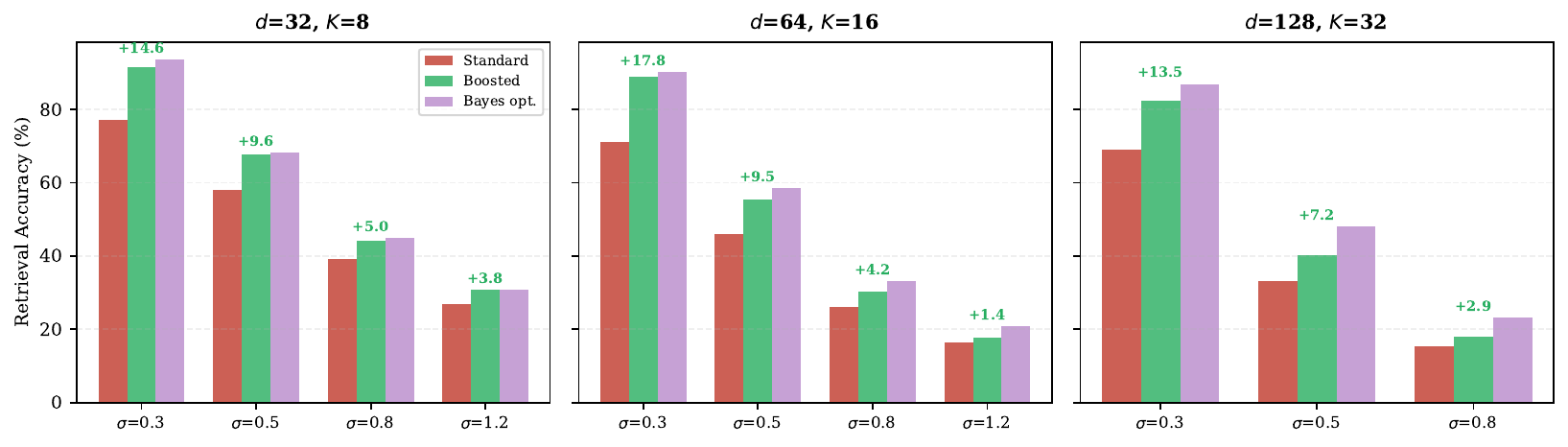}
\caption{Retrieval accuracy for standard vs.\ gradient-boosted attention ($M{=}2$, MLP gate) across configurations, with the Bayes-optimal ceiling shown for reference. Numbers above bars indicate the accuracy improvement in percentage points. Boosted attention closes most of the gap to the Bayes-optimal ceiling, especially at moderate noise levels ($\sigma \leq 0.5$).}
\label{fig:scaling_ablation}
\end{figure}

\section{Key/Value Source Ablation}
\label{app:kv_source}

An alternative design derives keys and values from the original input $\mathbf{x}$ rather than the residual $\mathbf{r}_m$:
\begin{equation}
    \mathrm{Attn}_m^{\text{(input)}}(\mathbf{r}_m) = \mathrm{softmax}\!\left(\frac{\mathbf{W}_Q^{(m)} \mathbf{r}_m \cdot (\mathbf{W}_K^{(m)} \mathbf{x})^\top}{\sqrt{d_h}}\right) \mathbf{W}_V^{(m)} \mathbf{x}.
\end{equation}
This ``K,V from input'' variant preserves the asymmetry between what the correction \emph{attends to} (the residual) and what it \emph{retrieves from} (the original context), analogous to cross-attention.
We compare the two designs on the WikiText-103 language modeling task using the same training setup as Section~\ref{sec:experiments}.

\begin{table}[h]
\centering
\caption{Effect of key/value source on WikiText-103 test perplexity. Deriving all projections from the residual outperforms the cross-attention variant by 1.4 points.}
\label{tab:kv_source}
\begin{tabular}{@{}lcc@{}}
\toprule
\textbf{K,V source} & \textbf{Test PPL} & \textbf{Val PPL} \\
\midrule
Residual $\mathbf{r}_m$ (ours) & $\mathbf{67.9}$ & $\mathbf{68.1}$ \\
Original input $\mathbf{x}$ & $69.3$ & $69.5$ \\
\bottomrule
\end{tabular}
\end{table}

Deriving all three projections from the residual yields 1.4 points lower perplexity (Table~\ref{tab:kv_source}).
We hypothesize that this is because the residual-derived keys and values allow the correction round to operate entirely within the residual subspace, producing value vectors $\mathbf{V}' = \mathbf{W}_V^{(1)}\mathbf{r}$ that are tailored to the error signal rather than the original representation.
This is consistent with Proposition~\ref{prop:separate}: since $\mathbf{V}' = \mathbf{W}_V^{(1)}\mathbf{r}$ spans a different subspace than $\mathbf{V}^{(0)} = \mathbf{W}_V^{(0)}\mathbf{x}$, the correction operates in a representation space distinct from round~0, whereas deriving values from $\mathbf{x}$ constrains both rounds to the same representational space.
Our implementation supports both variants via a \texttt{kv\_source} flag.

\section{Convex Hull Distance Computation}
\label{app:convex_hull}

To quantify the convex hull escape reported in Section~\ref{sec:analysis}, we compute the exact $\ell_2$ distance from the boosted output to $\mathrm{conv}(\mathbf{V}^{(0)})$ for individual attention heads.
Let $\mathbf{o} \in \mathbb{R}^{d_h}$ denote the boosted output for a single head at position $t$, and let $\mathbf{v}_1^{(0)}, \ldots, \mathbf{v}_t^{(0)} \in \mathbb{R}^{d_h}$ denote the round-0 value vectors for the same head at positions $1$ through $t$ (respecting the causal mask).
We find the nearest point in the convex hull by solving:
\begin{equation}
    d\!\left(\mathbf{o},\; \mathrm{conv}(\mathbf{V}^{(0)})\right) = \min_{\boldsymbol{\alpha} \in \mathbb{R}^t} \left\| \mathbf{o} - \sum_{i=1}^{t} \alpha_i \, \mathbf{v}_i^{(0)} \right\|_2 \quad \text{s.t.} \quad \alpha_i \geq 0 \;\;\forall i, \quad \sum_{i=1}^{t} \alpha_i = 1.
\end{equation}
This is a convex quadratic program (QP) with $t$ variables, one equality constraint, and $t$ nonnegativity constraints.
We solve it using Sequential Least-Squares Programming (SLSQP) via \texttt{scipy.optimize.minimize}, initialized at the uniform distribution $\alpha_i = 1/t$.
For computational tractability, we restrict to positions $t \leq 64$, limiting each QP to at most 64 variables.
For each of the four transformer layers, we uniformly sample 600 (position, head) pairs from 30 randomly selected test sequences.
In our experiments, $d_h = d / n_h = 256 / 4 = 64$, so each QP operates in $\mathbb{R}^{64}$.

\end{document}